\documentclass{article}
\usepackage{iclr2027_conference,times}

\usepackage{amsmath,amsfonts,bm}

\def\eqref#1{equation~\ref{#1}}

\def\1{\bm{1}}

\DeclareMathAlphabet{\mathsfit}{\encodingdefault}{\sfdefault}{m}{sl}
\SetMathAlphabet{\mathsfit}{bold}{\encodingdefault}{\sfdefault}{bx}{n}

\usepackage{enumitem}
\usepackage{graphicx}
\usepackage{booktabs}
\usepackage{amsfonts}
\usepackage{amsmath}
\usepackage{amssymb}
\usepackage{amsthm}
\usepackage{tikz}
\usepackage{url}
\usepackage{hyperref}
\usepackage{subcaption}

\newtheorem{definition}{Definition}
\newtheorem{proposition}{Proposition}
\newtheorem*{proposition*}{Proposition}

\title{Considering Context: When World Models Need Context Encoders}

\author{%
  Oleg Smirnov \\
  King AI Labs, Microsoft Gaming \\
  \texttt{oleg.smirnov@microsoft.com} \\
  \And
  Sofiane Ennadir \\
  King AI Labs,  Microsoft Gaming \\
  \texttt{sofiane.ennadir@king.com} \\
  \And
  John Pertoft \\
  King AI Labs, Microsoft Gaming \\
  \texttt{john.pertoft@king.com} \\
  \And
  Bjartur Hjaltason \\
  King AI Labs, Microsoft Gaming \\
  \texttt{bjartur.hjaltason@king.com} \\
  \And
  Sara Karimi \\
  King AI Labs, Microsoft Gaming \& KTH Royal Institute of Technology\\
  \texttt{sara.karimi@king.com} \\
}

\iclrfinalcopy

\begin{document}

\maketitle

\lhead{Preprint}
\renewcommand{\headrulewidth}{0pt}

\begin{abstract}
Methods for generalization in model-based reinforcement learning typically
assume that an agent cannot recover the latent context governing the
environment dynamics from its own experience, and therefore supplies it
externally. We formalize and test this assumption with \emph{predictive
sufficiency}, which quantifies what access to the context adds to next-step
prediction under the visitation distribution an agent induces, and separates
that quantity into a history-recoverable part and a residual requiring the
true context, both distinct from a finite model's deficit. We classify
context-aware algorithms by the predictive risk their conditioning set can
target and demonstrate across environments of increasing identification
difficulty that the headroom does not follow the MDP class. The same task
under different priors leaves predictive headroom in one setting and nothing
distinguishable from zero in another, where the agent's behavior implicitly
identifies the context and any benefit of such a mechanism cannot be
attributed to missing one-step predictive information. Where headroom
persists, the learned state exposes it only partially, and adding the true
context still lowers the risk. Our contribution is an analytical framework,
instantiated by a practical criterion for matching contextual mechanisms to
the predictive information available to them, estimated from the ordinary
trained agent.
\end{abstract}

\section{Introduction} \label{sec:introduction}
In model-based reinforcement learning (RL), even minor changes in environment
dynamics can degrade next-state predictions and, through compounding rollout
errors, the policies built on them~\citep{lee2020context}. Such variation is
often formalized as a contextual MDP~\citep{hallak2015contextual} or a hidden
parameter MDP~\citep{doshi2016hidden}. The latent parameters may be drawn
i.i.d.\ from a shared prior~\citep{benjamins2022contextualize} or evolve as a
background process, as in dynamic parameter
MDPs~\citep{pmlr-v139-xie21c}. They may change between episodes, within one,
or at every timestep~\citep{dulacarnold2021challenges}. The common approach is
to supply the context to world model through supervised system identification,
meta-learned adaptation, latent context encoders, or conditioning on the
ground-truth
context~\citep{duan2016rl,nagabandi2019learning,lee2020context,zintgraf2020varibad,
  evans2022context,roder2025dynamics,prasanna2024dreaming}.

\begin{figure}[!ht]
\centering
\includegraphics{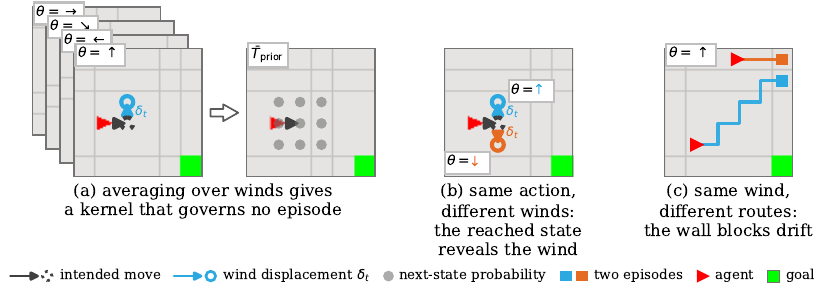}
\caption{A hidden per-episode wind $\theta$ probabilistically displaces the
  MiniGrid agent after each move. (a) Averaging the wind-specific kernels
  gives the prior-averaged kernel $\bar{T}_{\mathrm{prior}}$, which spreads
  probability over the neighboring cells and governs no episode. (b) The same
  action under two winds reaches different states, so the resulting state
  identifies the context without an explicit context estimate. (c) Under one
  wind the blue route bends in open cells while the wall blocks drift along
  the orange route: visited locations determine the evidence.}
\label{fig:wind_overview}
\end{figure}

The agent's trajectory, however, can already carry sufficient evidence about
the active context, and a model that uses history has access to it. Consider a
MiniGrid agent in Figure~\ref{fig:wind_overview}, where a hidden wind
displaces the agent after each move and walls block the displacement. In panel
(a), averaging the wind-specific kernels at one state and action gives a
mixture with no net displacement and inflated spread, which governs no
episode. However, under the same action, different winds can carry the agent
to different cells, so the reached state identifies the wind without an
explicit estimate, as panel (b) shows. A context encoder then has little
predictive information to add there. Moreover, panel (c) shows that the
identification depends on where the agent navigates: the wind bends a path
through open space, while the wall blocks the drift along a path beside it.

We focus on the prediction side of model degradation that context variation
causes. In this sense, a world model needs a context encoder when the latter
has one-step predictive information to add. We formalize this condition
through \emph{predictive sufficiency}, which asks a representation to preserve
the transition information the context carries rather than the context itself,
and quantify its gap as the Bayes one-step advantage of observing the context
for various predictors. We distinguish the prior-averaged structural variation
from \emph{predictive headroom}, defined as the potential improvement in
next-state prediction from the true context under the agent's observed
trajectories. We further split the headroom into a part recoverable from
history and a residual requiring the context, and separate both from the
learned predictor's finite-model deficit.

Our contribution is threefold: i) we analyze visitation-mediated implicit
identification, where the agent's behavior concentrates the context posterior
until no predictive information remains for another mechanism; ii) we
formulate predictive sufficiency and decompose its gap exactly across
conditioning sets (state and action, history, context); and iii) we propose a
criterion that determines which of a world model's conditioning sets have
one-step predictive information to add under the trained agent's visitation,
separately from return, and whether linear decoding of a learned state
recovers it, confirmed in experiments across MDP classes of increasing
identification difficulty.

\section{Related Work} \label{sec:related_work}
\textbf{MDPs with latent context variables.} Contextual MDP
formulation~\citep{hallak2015contextual} also appears as
\emph{hidden-parameter}~\citep{doshi2016hidden},
\emph{latent}~\citep{kwon2021rl}, and is closely related to
\emph{Bayes-adaptive}~\citep{duff2002optimal} MDPs, which consider a belief
over the unknown parameters. Work in these settings has studied
exploration~\citep{henaff2023study}, adaptation, regret, and zero-shot
generalization~\citep{kirk2023survey}. Common benchmarks in gaming
settings~\citep{chevalier2023minigrid,cobbe2020procgen,hafner2022benchmarking}
typically vary per-episode initial state distribution rather than dynamics,
while meta-RL tasks~\citep{rakelly2019pearl} mix reward, goal, and dynamics
variation. The Real-World Reinforcement Learning
suite~\citep{dulacarnold2021challenges} includes dynamics randomization,
stochastic disturbances, and degradations in continuous control, varied
between episodes. CARL~\citep{benjamins2022contextualize} turns continuous
environment parameters into explicit context variables and defines protocols
on contextualized control tasks. LILAC~\citep{pmlr-v139-xie21c} instead
considers structured evolution of an episode-level latent across a sequence of
tasks. We use these settings to study transition prediction and ask whether
and how predictive differences carry to downstream control.

\textbf{Context inference and conditioning.} In meta-RL,
PEARL~\citep{rakelly2019pearl} infers a latent task variable with a
permutation-invariant encoder and VariBAD~\citep{zintgraf2020varibad}
maintains a Bayes-adaptive belief over task parameters, separating information
shared across tasks from information inferred within one, in the spirit of the
empirical-Bayes recasting of meta-learning~\citep{grant2018recasting}. For
dynamics, UP-OSI~\citep{yu2017preparing} learns a supervised map from recent
state--action histories to physical parameters that condition a universal
policy, GrBAL and ReBAL~\citep{nagabandi2019learning} meta-train forward
models that adapt from recent trajectory segments through gradient-based and
recurrent updates respectively, and CaDM~\citep{lee2020context} and
IIDA~\citep{evans2022context} infer a latent dynamics context from recent
transitions without observing the true parameters. Within Dreamer,
cRSSM~\citep{prasanna2024dreaming} conditions the world model and policy on
observed context, while DALI~\citep{roder2025dynamics} infers it with a
separate history encoder and forward-dynamics objective. While these methods
offer inductive biases for inference from experience, we ask a prior question,
namely what information their inputs can supply under a given visitation,
which our framework makes precise. Relative to \emph{objective
mismatch}~\citep{lambert2020objective}, we report return beside probe estimates
of the Bayes predictive gap between conditioning sets, with finite-model error
removed from the predictive side and probe fitting error remaining. Under a
common reward, exact predictive sufficiency on the visited support is
\emph{value equivalence}~\citep{grimm2020value} with respect to all functions
on that support, whereas we ask which conditioning set suffices.

\section{Theoretical Framework} \label{sec:framework}
\textbf{Preliminaries.} We formalize the environment as a contextual MDP $(S,
A, T_\theta, R_\theta, \gamma)$, where both the transition kernel $T_\theta$
and the reward function $R_\theta$ are parameterized by a latent context
variable $\theta$. We denote an episode trajectory by
$\tau=(s_0,a_0,s_1,a_1,\ldots,s_T)$ and its prefix up to time $t$ by
$\tau_{\leq t}$. The context is drawn at the start of each episode, either
independently from a prior $p(\theta \mid \alpha)$ with shared hyperparameters
$\alpha$, or from an autonomous Markov chain $p(\theta_{i+1} \mid \theta_i)$
across episodes. We restrict attention to contexts fixed within an episode:
time-varying context requires an augmented latent process and remains outside
the present framework. Our analysis is concerned with the transition kernel,
and its predictive quantities require no reward labels.

Let $p^\pi(\theta,s,a,s')$ denote the joint distribution induced by drawing
$\theta\sim p(\theta\mid\alpha)$, rolling out policy $\pi$ under $T_\theta$,
and sampling a transition according to the evaluation weighting. Assume the
transition process is Markov conditional on $\theta$, so that
$s_{t+1}\perp\tau_{<t}\mid (s_t,a_t,\theta)$. For a conditioning set $X$,
$\mathcal{R}_X$ denotes the minimum expected one-step log loss of a predictor
observing $X$ under $p^\pi$. Table~\ref{tab:interfaces} summarizes the
hierarchy of conditioning sets we analyze, in the order of its rows: the
memoryless $(s_t,a_t)$, within-episode history, cross-episode history, and the
Oracle context, with the window and recurrent classes sharing the
within-episode level.

\begin{table}[t]
\centering
\scriptsize
\setlength{\tabcolsep}{4pt}
\caption{Conditioning hierarchy of this section and the method classes it
  orders, by the predictive risk their conditioning set can target. Bold
  denotes the representatives evaluated in Section~\ref{sec:experiments}. The
  memoryless class is instantiated by the probe fitted on $(s_t,a_t)$.}
\label{tab:interfaces}
\begin{tabular}{@{}llll@{}}
\toprule
Class & Conditioning set (bound) & Sufficient when & Methods \\
\midrule
Memoryless & $(s_t,a_t)$ \,($\mathcal{R}_{\mathrm{M}}$)
  & $\theta$ decodable from $(s_t,a_t)$
  & \textbf{Memoryless probe} \\
Transition window & last $K$ transitions \,($\geq\mathcal{R}_{\mathrm{H}}$)
  & recent evidence
  & \textbf{DALI-S}, CaDM, UP-OSI, GrBAL, PEARL, IIDA \\
Recurrent & within-episode summary \,($\geq\mathcal{R}_{\mathrm{H}}$)
  & evidence within the episode
  & \textbf{DreamerV3}, ReBAL, VariBAD \\
Cross-episode chain & cross-episode summary \,($\geq\mathcal{R}_{\mathrm{H}^{+}}$)
  & episodes correlated
  & \textbf{Carried state}, \textbf{LILAC}, RL$^2$, L2RL \\
Ground-truth context & realized $\theta$ \,($\mathcal{R}_{\mathrm{O}}$)
  & always
  & \textbf{cRSSM-D}, cRSSM-C \\
\bottomrule
\end{tabular}
\end{table}

\textbf{Oracle context $(s_t,a_t,\theta)$.} We first define the Oracle
observer of $(s_t,a_t,\theta)$, the reference every level is compared
with. The Oracle conditions on the realized context $\theta$, so its target is
$T_\theta$ and it attains the risk $\mathcal{R}_{\mathrm{O}}$. In a partially
observed process the Oracle must also condition on history. For every level,
we call \emph{headroom} the gap between that level's minimum risk and
$\mathcal{R}_{\mathrm{O}}$.

\textbf{Memoryless $(s_t,a_t)$.} The memoryless level conditions on the
current state--action pair only. In the realizable infinite-data limit, such a
world model trained by maximum likelihood on the pooled transitions converges
to
\begin{equation}
\bar T_{\mathrm{emp}}(s' \mid s, a) = p^\pi(s' \mid s, a) = \mathbb{E}_{\theta
  \sim p^\pi(\theta \mid s, a)} \bigl[T_\theta(s' \mid s, a)\bigr].
\label{eq:empirical-marginal}
\end{equation}
By Bayes' rule, $p^\pi(\theta\mid s,a)\propto
p(\theta\mid\alpha)d^\pi_\theta(s,a)$, where $d^\pi_\theta$ is the
corresponding state--action occupancy induced by $T_\theta$ and $\pi$. Thus it
need not equal $p(\theta\mid\alpha)$, making $\bar T_{\mathrm{emp}}$ an
observational mixture and the Bayes-optimal log-loss predictor conditioned
only on $(s,a)$, with minimum risk $\mathcal{R}_{\mathrm{M}}$.

The difference between the memoryless risk $\mathcal{R}_{\mathrm{M}}$ and the
Oracle risk $\mathcal{R}_{\mathrm{O}}$ is
\begin{equation}
\mathcal{I}_{\mathrm{emp}} = I(\theta; s_{t+1} \mid s_t, a_t) =
\mathbb{E}_{(s,a) \sim p^\pi}\, \mathbb{E}_{\theta \sim p^\pi(\theta \mid
  s,a)} \left[D_{\mathrm{KL}}\!\left( T_\theta(\cdot \mid s,a) \,\middle\|\,
  \bar T_{\mathrm{emp}}(\cdot \mid s,a)\right)\right].
\label{eq:mi}
\end{equation}
We call $\mathcal{I}_{\mathrm{emp}}$ the \emph{predictive headroom}. It is the
exact reduction in one-step log loss from revealing $\theta$ to a predictor
conditioned on $(s_t,a_t)$, such that small values leave little advantage even
for a memoryless Oracle, while large values leave a substantial one. A finite
memoryless model can also incur additional estimation, approximation, or
optimization error. In discrete environments with repeated contexts,
$\mathcal{I}_{\mathrm{emp}}$ can be estimated from context-labeled transition
counts, whereas continuous settings require mutual information
estimation. Proposition~\ref{prop:sim} in Appendix~\ref{app:sim-proof}, stated
for a common bounded reward, bounds the value error under any stationary
Markov policy by $\gamma R_{\max}/(1-\gamma)^2$ times the supremum over
$(s,a)$ of the transition error~\citep{kearns2002near, brafman2002rmax}. A
visitation-averaged headroom does not enter this supremum, so we evaluate
predictive headroom and return separately.

At a fixed query $(s,a)$, averaging the conditional kernels under the context
prior gives
\begin{equation}
\bar T_{\mathrm{prior}}(s' \mid s, a) = \mathbb{E}_{\theta \sim p(\theta \mid
  \alpha)}\bigl[T_\theta(s' \mid s, a)\bigr],
\label{eq:bma}
\end{equation}
the Bayesian model average~\citep{hoeting1999bayesian}. We refer to this
mixture's predictive loss relative to the context-conditional kernels as the
\emph{marginal averaging problem}. This standardized reference compares
kernels at a shared $(s,a)$ before visitation reweights the context prior. If
$\theta\perp(s,a)$ under $p^\pi$, it equals $\bar T_{\mathrm{emp}}$. For a
$\theta$-independent reference distribution $\mu(s,a)$, define
\begin{equation}
\mathcal{I}_{\mathrm{struct}} = \mathbb{E}_{(s,a) \sim \mu}\,
\mathbb{E}_{\theta \sim p(\theta \mid \alpha)} \left[D_{\mathrm{KL}}\!\left(
  T_\theta(\cdot \mid s,a) \,\middle\|\, \bar T_{\mathrm{prior}}(\cdot \mid
  s,a)\right)\right].
\label{eq:mi-struct}
\end{equation}
$\mathcal{I}_{\mathrm{struct}}$ is conditional mutual information under the
constructed joint $\mu(s,a)p(\theta\mid\alpha)T_\theta(s'\mid s,a)$ and
averages the corresponding Jensen gap across shared queries
(Appendix~\ref{app:sim-proof}). This quantity measures kernel separation
rather than the training loss of a pooled predictor.
$\mathcal{I}_{\mathrm{emp}}$ uses the query distribution and context weights
induced by data collection, whereas $\mathcal{I}_{\mathrm{struct}}$ averages
over $\mu$. The two quantities therefore coincide when $\mu(s,a)=p^\pi(s,a)$
and the independence condition above holds, but neither bounds the other in
general. The difference
$\mathcal{I}_{\mathrm{struct}}-\mathcal{I}_{\mathrm{emp}}$ is, therefore,
descriptive, and may be negative.

If $p^\pi(\theta\mid s,a)$ is concentrated, the current $(s,a)$ already
identifies much of the context, so large structural variation can coexist with
little predictive headroom. We call this \emph{visitation-mediated implicit
identification}. However, context remains informative for prediction when
visitation overlaps across contexts and the posterior over $\theta$ stays
diffuse over contexts whose kernels differ at the visited queries. In
practice, measuring how well $\theta$ can be decoded from the
current $(s,a)$ reveals to what extent the state and action already identify
the context.

\textbf{Within-episode history $(\tau_{\leq t},a_t)$.} Within an episode the
predictor may condition on the full prefix $\tau_{\leq t}$ together with
$a_t$, and its minimum risk is $\mathcal{R}_{\mathrm{H}}$. The window and
recurrent classes condition on a segment or a summary of this prefix, so their
risks are bounded below by $\mathcal{R}_{\mathrm{H}}$. The Bayes-optimal
history-aware target is
\begin{equation}
\bar T_{\mathrm{hist}}(s' \mid \tau_{\leq t}, a_t) = \mathbb{E}_{\theta \sim
  p^\pi(\theta \mid \tau_{\leq t},a_t)} \bigl[T_\theta(s' \mid s_t,a_t)\bigr].
\label{eq:history-marginal}
\end{equation}
The residual Oracle headroom after observing history is
\begin{equation}
\mathcal{I}_{\mathrm{hist}} = I(\theta; s_{t+1} \mid \tau_{\leq t},a_t) =
\mathcal{R}_{\mathrm{H}}-\mathcal{R}_{\mathrm{O}}.
\label{eq:mi-history}
\end{equation}
Further, the chain rule of mutual information decomposes
$\mathcal{I}_{\mathrm{emp}}$ into the headroom closed by history and the
residual Oracle headroom.
\begin{proposition}[Predictive headroom decomposition]
\label{prop:headroom}
Under the conditional-Markov assumption and a common transition distribution,
\begin{equation}
\mathcal{I}_{\mathrm{emp}} = \mathcal{R}_{\mathrm{M}}-\mathcal{R}_{\mathrm{O}}
= \underbrace{\mathcal{R}_{\mathrm{M}}-
  \mathcal{R}_{\mathrm{H}}}_{\text{headroom closed by history}} +
\underbrace{\mathcal{I}_{\mathrm{hist}}}_{\text{residual Oracle headroom}}
\label{eq:headroom-decomposition}
\end{equation}
Consequently, $0 \leq \mathcal{I}_{\mathrm{hist}} \leq
\mathcal{I}_{\mathrm{emp}}$.
\end{proposition}

The proof appears in Appendix~\ref{app:sim-proof}. The first term measures
predictive information recovered from trajectory history beyond
$(s_t,a_t)$. The second is the improvement still available from revealing
$\theta$ to an ideal history-aware predictor. When
$\mathcal{I}_{\mathrm{hist}}=0$, context provides no further advantage after
history. For any finite recurrent model $q$, the remaining Oracle gap is
\begin{equation}
\mathcal{R}(q)-\mathcal{R}_{\mathrm{O}} =
\bigl[\mathcal{R}(q)-\mathcal{R}_{\mathrm{H}}\bigr] +
\mathcal{I}_{\mathrm{hist}},
\label{eq:finite-model-gap}
\end{equation}
where the bracketed term is estimation, approximation, or optimization error.
Neither a context pathway computed from the same within-episode history nor a
different recurrent core can improve on $\mathcal{R}_{\mathrm{H}}$ in
principle, but either may help a finite model approach it.

For a $k$-step rollout, let $\tau_{t+1:t+k}=(s_{t+1},a_{t+1},\ldots,s_{t+k})$
denote the next $k$ steps generated under $\pi$ from the prefix $(\tau_{\leq
  t},a_t)$. Under the conditional-Markov assumption, for a policy that acts on
the history alone, the chain rule gives
\begin{equation}
I\bigl(\theta;\,\tau_{t+1:t+k}\mid\tau_{\leq t},a_t\bigr)
= \sum_{j=0}^{k-1} I\bigl(\theta;\,s_{t+j+1}\mid\tau_{\leq t+j},a_{t+j}\bigr),
\label{eq:rollout-chain}
\end{equation}
where each term averages, over the rollout's prefixes, the one-step
conditional mutual information at the visited $(s_{t+j},a_{t+j})$ under the
context weights $p^\pi(\theta\mid\tau_{\leq t+j},a_{t+j})$ of the rollout
itself. The $k$-step information therefore vanishes along a rollout whenever
the one-step quantity is zero at every step. Imagined rollouts leave the
visited support, and the compounding of finite-model error along them is a
property of the learned model.

\textbf{Cross-episode history $(\tau^{+}_{\leq t},a_t)$.} The decomposition
above takes $\theta$ independent across episodes, so histories end at episode
boundaries. When contexts instead evolve as a Markov chain, $\theta_{i+1}\sim
p(\theta_{i+1}\mid\theta_i)$, completed episodes carry information about the
current context~\citep{choi2001hidden,pmlr-v139-xie21c}. For an ergodic chain,
the stationary distribution plays the role of the prior. Let $\tau^{+}_{\leq
  t}$ denote the cross-episode history concatenating completed episodes with
the current prefix, and $\mathcal{R}_{\mathrm{H}^{+}}$ the corresponding
minimum risk. Since $\tau_{\leq t}\subseteq\tau^{+}_{\leq t}$,
\begin{equation}
\mathcal{I}_{\mathrm{emp}}
= \underbrace{\mathcal{R}_{\mathrm{M}}-\mathcal{R}_{\mathrm{H}}}_{\text{within-episode}}
+ \underbrace{\mathcal{R}_{\mathrm{H}}-\mathcal{R}_{\mathrm{H}^{+}}}_{\text{cross-episode}}
+ \underbrace{I\bigl(\theta;\, s_{t+1}\mid \tau^{+}_{\leq t}, a_t\bigr)}_{\text{residual}},
\label{eq:chain-decomposition}
\end{equation}
where the middle term vanishes under independent draws, which motivates
episode-start evaluation before within-episode evidence accumulates. The
in-principle bound of Equation~\ref{eq:finite-model-gap} is relative to the
information interface: $\mathcal{R}_{\mathrm{H}}$ bounds any predictor
computed from the current episode, while a recurrent state carried across
resets or an explicit model of the context chain targets
$\mathcal{R}_{\mathrm{H}^{+}}\leq\mathcal{R}_{\mathrm{H}}$.

\textbf{Predictive sufficiency.} The hierarchy fixes the risk each
conditioning set can reach, but it leaves open which representation of that
set attains it. We define \emph{predictive sufficiency} as the requirement
that a context representation preserve the transition information carried by
$\theta$ rather than $\theta$ itself.

\begin{definition}[$\varepsilon$-predictive sufficiency]
\label{def:psuff}
For a data-collection policy $\pi$, let $\tau$ be generated under episode
context $\theta\sim p(\theta\mid\alpha)$. A representation $c_t$ computed from
$\tau_{\leq t}$ induces the Bayes-optimal predictor
\begin{equation}
\bar T_c(s' \mid s_t,a_t,c_t) = \mathbb{E}_{\theta' \sim p^\pi(\theta' \mid
  s_t,a_t,c_t)} \bigl[T_{\theta'}(s' \mid s_t,a_t)\bigr].
\label{eq:representation-marginal}
\end{equation}
The representation is $\varepsilon$-predictively sufficient under $\pi$ if
\begin{equation}
\mathbb{E}_{\theta \sim p(\theta \mid \alpha)}\,\mathbb{E}_{\tau \sim \pi,
  T_\theta}\!\left[\frac{1}{T}\sum_{t=0}^{T-1}
  D_{\mathrm{KL}}\!\left(T_\theta(\cdot \mid s_t,a_t)\,\middle\|\, \bar
  T_c(\cdot \mid s_t,a_t,c_t)\right)\right] \leq \varepsilon.
\label{eq:psuff}
\end{equation}
\end{definition}

A constant $c_t$ induces $\bar T_{\mathrm{emp}}$ and attains the sufficiency
gap $\varepsilon=\mathcal{I}_{\mathrm{emp}}$. Preserving the full
posterior-predictive information in $\tau_{\leq t}$ attains
$\varepsilon=\mathcal{I}_{\mathrm{hist}}$, while $\varepsilon=0$ requires
identifying the realized transition kernel only on the visited support.
Contexts inducing the same kernel there are predictively equivalent and may
share a representation. Predictive sufficiency retains trajectory information
that affects transitions and may discard other context components. In the
correlated-episode setting, $c_t$ may be computed from $\tau^{+}_{\leq t}$.
The definition then applies with the expectation taken over the stationary
chain, and preserving the full cross-episode posterior attains the residual of
Equation~\ref{eq:chain-decomposition}.

The decomposition identifies the part of an Oracle gap accessible to
architectural intervention, and as a criterion it attributes one-step
predictive information to the world model's conditioning set. Its control
effect remains an empirical question: return also depends on horizon and
reward sensitivity, and the definition is relative to the $(s, a)$
distribution induced by $\pi$, so policy improvement requires coverage of the
new distribution or a fresh estimate of $\varepsilon$.

\section{Controlled Diagnostic} \label{sec:diagnostic}
The framework distinguishes variation in a transition family from predictive
headroom under the distribution generated by an agent. We test this
distinction in tractable i.i.d.\ settings where the context-conditional
kernels are known. MiniGrid~\citep{chevalier2023minigrid} provides a finite
state space in which both $\mathcal{I}_{\mathrm{struct}}$ and
$\mathcal{I}_{\mathrm{emp}}$ can be evaluated from the exact kernels at the
logged continuous contexts. The Point-Mass-Attractor environment provides a
continuous-state counterpart with known noisy one-step kernels.

\textbf{Diagnostic design.} We implement two MiniGrid transition variants:
MiniGrid-Wind samples a drift that changes the realized displacement after
each action, while MiniGrid-Resistance samples an action-failure
probability. Point-Mass-Attractor (Figure~\ref{fig:attractor_force_field},
Appendix~\ref{app:attractor_diagnostic}) samples one of the attractor
locations, causing trajectories to curve toward the selected center. In all
cases, the diagnostic varies only an unobserved episode-level transition
parameter. We collect trajectories under uniform random actions and record
context for estimation while withholding it from the predictive information
set. Both continuous and discrete settings exploit the known kernels rather
than estimating next-state distributions from samples. Environment definitions
and protocols appear in Appendix~\ref{app:controlled_diagnostics}.

Alongside the fixed $\theta$-independent reference, we evaluate
$\mathcal{I}_{\mathrm{struct}}^\pi$ by setting $\mu(s,a)=p^\pi(s,a)$ while
retaining prior context weights. Comparing $\mathcal{I}_{\mathrm{struct}}^\pi$
with $\mathcal{I}_{\mathrm{emp}}$ holds the query distribution fixed and
changes only the context weights. All information quantities use natural
logarithms.

\begin{figure}[t]
\centering
\includegraphics[width=\linewidth]{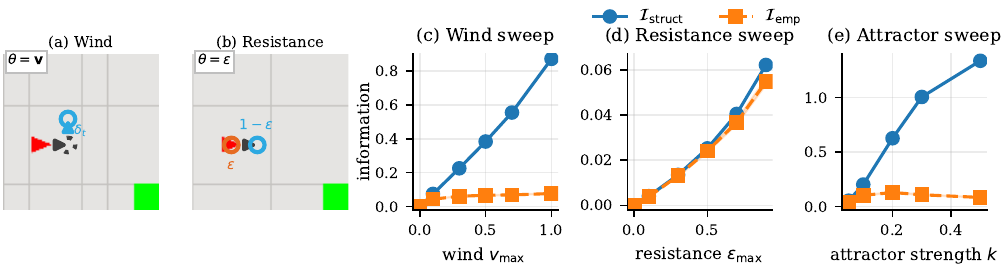}
\caption{Controlled diagnostics. Wind applies a sampled drift after the native
  action (a), whereas Resistance either commits or voids the action (b).
  Across perturbations (c--e), $\mathcal{I}_{\mathrm{struct}}$ separates from
  predictive headroom on Wind and Point-Mass-Attractor but remains close to it
  on Resistance.}
\label{fig:diagnostic-sweeps}
\end{figure}

\begin{table}[t]
\centering
\scriptsize
\caption{Representative MiniGrid and Point-Mass-Attractor settings. The
  matched column sets $\mu(s,a)=p^\pi(s,a)$ while retaining prior context
  weights. CIs are 95\% episode-level bootstrap. All estimates are
  conservatively biased upward, so finite-sample error can inflate headroom.}
\label{tab:diagnostic-summary}
\begin{tabular}{@{}llccc@{}}
\toprule
State & Setting & $\mathcal{I}_{\mathrm{struct}}$
  & $\mathcal{I}_{\mathrm{struct}}^\pi$
  & $\mathcal{I}_{\mathrm{emp}}$ \\
\midrule
Discrete & Wind ($v_{\max}=1$) & 0.871 & 0.620 & $0.077 \pm 0.002$ \\
Discrete & Resistance ($\epsilon_{\max}=0.9$) & 0.062 & 0.059
  & $0.055 \pm 0.001$ \\
Continuous & Attractor ($k=0.3$) & 1.006 & 1.006
  & $0.109 \pm 0.002$ \\
\bottomrule
\end{tabular}
\end{table}

\textbf{Structural variation and predictive headroom.} Panels c--e in
Figure~\ref{fig:diagnostic-sweeps} depict two regimes. Wind develops
substantial sensitivity while its estimated predictive headroom is 8.8\% of
$\mathcal{I}_{\mathrm{struct}}$. At $v_{\max}=1$, matching the query
distribution reduces the reference by 29\%, and the data-conditional context
weights account for a further 88\% reduction. Resistance changes little across
the three columns of Table~\ref{tab:diagnostic-summary}; its low headroom
follows from modest kernel separation rather than state-mediated
identification. Point-Mass-Attractor reproduces the Wind regime in continuous
state, where its kernel separation is invariant to the query distribution on
the interior support, but an estimated 11\% remains under observed visitation.

\textbf{Intervening on visitation.} To demonstrate the dependency, we reset
selected state components during collection without resampling the episode
context. Resetting MiniGrid position and direction every ten steps raises
$\mathcal{I}_{\mathrm{emp}}$ by 430\%, and resets every two steps raise it by
833\%. In Point-Mass-Attractor, resetting position and velocity every ten
steps raises the estimate by 324\%; preserving velocity raises it by only
28\%. Hence, the control identifies velocity as the principal carrier of
context in the continuous system. Appendix~\ref{app:controlled_diagnostics}
reports the full sweeps and complementary state-projection checks.

Together, these diagnostics show why contextual variation alone does not
warrant contextual modeling. Wind and Point-Mass-Attractor have strongly
separated kernels, but their induced states leave little one-step headroom,
whereas MiniGrid-Resistance reaches similar headroom because its kernels
differ little.

\section{Experimental Validation} \label{sec:experiments}
Based on the theoretical analysis, we predict that predictive sufficiency
under a given visitation determines when a mechanism has additional
information to exploit: headroom must remain, and conditioning on its inputs
must lower the corresponding optimal risk. The resulting bounds separate the
available information from how effectively the inductive bias uses it, which
depends on the information structure rather than the class the mechanism was
designed for.

Based on those insights, we formulate four falsifiable predictions:
\begin{itemize}[leftmargin=*]
  \item \textbf{(P1)} Implicit identification renders observed context
    redundant for one-step prediction irrespective of the MDP class. When
    visitation makes the context identifiable from the current state and
    action, the agent's own behavior has already performed the identification
    that context mechanisms are built for, making the predictive headroom
    vanish. Benefits of the mechanism therefore cannot be attributed to
    missing predictive information.
  \item \textbf{(P2)} Predictive sufficiency separates information failure
    from finite-model error. When headroom persists but the episode's history
    reaches the observed-context risk, any remaining predictive deficit for a
    model with access to that history is due to estimation, approximation, or
    optimization.
  \item \textbf{(P3)} The sufficiency optimum is relative to its conditioning
    set. For methods that do not observe context, only mechanisms that carry
    information across episode boundaries can improve on the within-episode
    optimum when consecutive contexts are dependent. Their predictive
    advantage should be largest at episode start and vanish under independent
    draws.
  \item \textbf{(P4)} For a fixed visitation distribution, the reduction in
    expected log loss obtained by adding a context representation to the
    current state and action equals the difference between predictive headroom
    and that representation's minimum sufficiency gap. Under implicit
    identification, both the reduction and the tolerance vanish.
\end{itemize}

\textbf{Method classes.} Table~\ref{tab:interfaces} groups representative
methods by their conditioning information, ordered by the predictive risk each
class can target. The window class adds a dedicated context encoder over the
last $K$ transitions, ordered or unordered. We take
DALI-S~\citep{roder2025dynamics} as the ordered representative, since none of
our settings makes transition order informative and the framework therefore
cannot separate the variants. For the cross-episode class, we evaluate a
Markov chain-latent agent similar to LILAC~\citep{pmlr-v139-xie21c} and a
carried-state variant that preserves recurrent state across resets, the
mechanism of RL$^2$~\citep{duan2016rl} and L2RL~\citep{wang2016learning}, and
the smallest member of this family, used as an ablation in the
non-i.i.d.\ setting only. Finally, the observed-context class provides the
privileged reference that separates recoverable headroom from task difficulty,
represented by cRSSM-D~\citep{prasanna2024dreaming}. All evaluated methods
share the same DreamerV3 backbone~\citep{hafner2025mastering}. Implementation,
training, and evaluation details are provided in
Appendix~\ref{app:experimental_details}.

\textbf{Environments.} Common benchmarks often vary initial states rather than
the latent dynamics, while meta-RL suites mix reward with dynamics variation,
and perturbation studies confound predictive headroom with task difficulty. To
address this, we build on CARL~\citep{benjamins2022contextualize} and
RWRL~\citep{dulacarnold2021challenges}, which isolate dynamics variation. As
with the methods, we select representative environments for the classes of
contextual MDPs studied in prior work, ordered by the difficulty of
identifying the context.

The two i.i.d.\ settings are based on DMC Walker from the CARL suite,
repeating the contrast of Section~\ref{sec:diagnostic} at benchmark scale: (i)
an \textit{identifiable} setting, where a latent horizontal wind, sampled
uniformly per episode, exerts force through nonzero fluid density, probe
headroom collapses while decodability stays high (P1, P4), and (ii) a
\textit{hard} setting, where a latent actuator-strength scale is sampled
uniformly under fixed gravity and headroom persists above the permutation null
(P2). The settings separate collapsed and persistent headroom under baseline
visitation. Each method is then profiled under its own. In the
non-i.i.d.\ drift setting, a single Walker parameter follows an RWRL
perturbation schedule, performing a random walk across episodes at three rates
on identical frozen realizations (P3).

Additionally, we include a simple linear-Gaussian point mass environment whose
hidden context evolves across episodes as a Markov chain with high
self-transition probability, so completed episodes carry predictive value
about the current one. There, known transition kernels and disturbance filters
provide analytical references for calibrating the Gaussian-score diagnostics
(Appendix~\ref{app:calibration}). A short-episode variant limits
within-episode evidence, and an independent-draw variant controls for the
information supplied by cross-episode disturbance history (P3). One difficulty
class remains uncovered: to our knowledge, no standard benchmark makes the
context identifiable only through deliberately informative actions, so every
passive class is predicted equivalent there. Table~\ref{tab:environments}
summarizes the settings with their context process and difficulty level.

\begin{table}[t]
  \centering
  \scriptsize
  \caption{Environment settings, ordered by the level at which the context
    becomes identifiable.}
  \label{tab:environments}
  \begin{tabular}{@{}ll p{0.44\linewidth} l l@{}}
    \toprule
    Setting & Environment & Context process & Difficulty level & Tests \\
    \midrule
    Calibration
      & Point mass
      & persistent two-state chain, i.i.d.\ and short-episode variants
      & within/across episodes
      & P3 (control) \\
    Identifiable
      & CARL Walker
      & i.i.d.\ per episode
      & current state and action
      & P1, P4 \\
    Hard
      & CARL Walker
      & i.i.d.\ per episode
      & episode history
      & P2, P4 \\
    Drift
      & RWRL Walker
      & non-i.i.d.\ Markov chain across episodes
      & across episodes
      & P3 \\
    \bottomrule
  \end{tabular}
\end{table}

\textbf{Measurements.} To estimate headroom, history recovery, and the
sufficiency gap, we compare held-out next-state log losses from MLP and
recurrent probes with Gaussian likelihood heads, fitted independently of the
agent on each method's evaluation episodes. Their inputs are (i) current state
and action, (ii) adding the true context, (iii) adding a linear readout $c_t$
trained to predict $\theta$ from the learned recurrent state, (iv) adding
both, (v) adding a fixed-size history window, and (vi) adding both window and
readout. A separate decodability probe predicts context from state and action;
a permutation null shuffles contexts across episodes. Comparisons use fixed
training budgets and separate episodes for fitting, validation, and scoring.

Because the probes are computed under each method's visitation, they are
considered as within-method evidence, following the policy dependence of
Definition~\ref{def:psuff}. Cross-method comparisons rely on return over
paired evaluation contexts and the class-level predictions, and, on the drift
setting, on the start transient. On the drift setting the primary metric is
the start transient: early-window prior loss minus the cached-prefix reference
loss.

\textbf{Evaluation design.} We use differences in held-out probe risks as
empirical diagnostics of the predictive quantities in
Section~\ref{sec:framework}. Let $\widehat{\mathcal{R}}_X$ denote the
next-state log loss of a fitted probe with inputs $X$.
Probe headroom is defined as $\widehat
G_\theta=\widehat{\mathcal{R}}_{\mathrm{M}}-
\widehat{\mathcal{R}}_{\mathrm{O}}$. Under the assumptions of
Section~\ref{sec:framework}, its population counterpart satisfies
$G_\theta=\mathcal{I}_{\mathrm{emp}}+e_{\mathrm{M}}-e_{\mathrm{O}}$, where
$e_X$ includes approximation and fitting error above the corresponding Bayes
risk. The finite-family reference $\widehat{\mathcal{R}}_{\mathrm{ref}}$ is
the minimum risk over the window, the learned-state readout, their
combination, and a recurrent probe. Where probe headroom separates from the
permutation null, normalized gain is
$(\widehat{\mathcal{R}}_{\mathrm{M}}-\widehat{\mathcal{R}}_c)/ \widehat
G_\theta$. Residual context gain is
$\widehat{\mathcal{R}}_c-\widehat{\mathcal{R}}_{c,\theta}$. For Bayes-optimal
predictors, the corresponding risk difference under the averaging law of
Definition~\ref{def:psuff} equals the smallest admissible $\varepsilon$. We
also report differences over the first $k{=}25$ transitions of each episode.

\textbf{Decision criteria.} We formulate our proposed criterion as
follows. The decisions are derived from differences of
$\widehat{\mathcal{R}}_X$, which estimate, with bootstrap intervals,
differences of the fitted probes' population risks
$Q_X=\mathcal{R}_X+e_X$. The history probe is the validation-selected window
or recurrent probe. Equation~\ref{eq:mi-history} and the nonnegativity of
$e_{\mathrm{H}}$ alone give $\mathcal{I}_{\mathrm{hist}}\le
Q_{\mathrm{H}}-Q_{\mathrm{O}}+e_{\mathrm{O}}$. A closed population gap shows
that history carries the information up to the Oracle probe's fitting error,
while an open gap leaves information missing from history indistinguishable
from information the probe fails to extract. Therefore, inferences from closed
gaps require no estimate of $e_{\mathrm{H}}$. Since history cannot raise the
Bayes risk, a history probe with higher risk than the memoryless probe has the
larger fitting error, assessed through the held-out difference without
ground-truth context $\theta$.

In cases where headroom does not separate from the null, its upper confidence
bound plus $e_{\mathrm{O}}$ bounds $\mathcal{I}_{\mathrm{emp}}$, and headroom
alone does not establish one-step predictive benefit beyond $(s_t,a_t)$. In
cases where headroom separates, the upper confidence endpoint of the estimated
$Q_{\mathrm{H}}-Q_{\mathrm{O}}$ plus $e_{\mathrm{O}}$ bounds
$\mathcal{I}_{\mathrm{hist}}$. Finally, when in addition the gap stays open,
its lower confidence endpoint positive, the outcome is unresolved. We note
that cross-method comparison does not separate cross-episode information from
each method's within-episode prediction quality.

\section{Results} \label{sec:results}
We report the measurements of Section~\ref{sec:experiments} on the three
settings of Table~\ref{tab:environments} and interpret each through the
criterion's cases.

\begin{table}[t]
\centering
\scriptsize
\caption{Computed quantities over i.i.d.\ environments. Identifiable and Hard
  settings report return, probe headroom, and the residual context gain of the
  linear context readout. Entries are three-seed IQM with 95\% bootstrap
  confidence intervals, method differences use paired seed resampling.}
\label{tab:results-iid}
\begin{tabular}{@{}lcccccc@{}}
\toprule
& \multicolumn{3}{c}{Identifiable} & \multicolumn{3}{c}{Hard} \\
\cmidrule(lr){2-4}\cmidrule(l){5-7}
Method
& \rotatebox{45}{Return $\uparrow$} & \rotatebox{45}{Headroom}
& \rotatebox{45}{Resid. gain}
& \rotatebox{45}{Return $\uparrow$} & \rotatebox{45}{Headroom}
& \rotatebox{45}{Resid. gain} \\
\midrule
DreamerV3     & $712 \pm 10$ & $-0.00 \pm 0.12$ & $-0.13 \pm 0.11$ & $946 \pm 6$ & $2.39 \pm 0.70$ & $0.38 \pm 0.08$ \\
DALI-S        & $756 \pm 72$ & $0.03 \pm 0.13$ & $-0.15 \pm 0.13$ & $941 \pm 14$ & $2.10 \pm 1.76$ & $0.44 \pm 0.19$ \\
LILAC-style   & $589 \pm 88$ & $-0.09 \pm 0.08$ & $0.00 \pm 0.09$ & $\mathbf{953 \pm 4}$ & $3.49 \pm 0.68$ & $0.73 \pm 0.29$ \\
cRSSM-D       & $\mathbf{862 \pm 17}$ & $-0.14 \pm 0.16$ & $-0.06 \pm 0.09$ & $948 \pm 6$ & $3.64 \pm 1.71$ & $-0.22 \pm 0.80$ \\
\bottomrule
\end{tabular}
\end{table}

\begin{table}[t]
\centering
\scriptsize
\caption{Computed quantities over the non-i.i.d.\ environment: start transient and
  start prior loss at three drift rates. Entries are means $\pm$ sample stdev
  across seeds.}
\label{tab:results-drift}
\begin{tabular}{@{}lcccccc@{}}
\toprule
& \multicolumn{3}{c}{Start transient $\downarrow$}
& \multicolumn{3}{c}{Start prior loss $\downarrow$} \\
\cmidrule(lr){2-4}\cmidrule(l){5-7}
Method
& \rotatebox{45}{Low} & \rotatebox{45}{Mid} & \rotatebox{45}{High}
& \rotatebox{45}{Low} & \rotatebox{45}{Mid} & \rotatebox{45}{High} \\
\midrule
DreamerV3     & $\mathbf{2.64 \pm 0.45}$ & $2.80 \pm 0.15$ & $2.79 \pm 0.86$ & $\mathbf{10.15 \pm 0.82}$ & $10.77 \pm 0.59$ & $11.24 \pm 0.83$ \\
Carried state & $4.06 \pm 0.36$ & $3.92 \pm 0.69$ & $3.93 \pm 0.51$ & $13.30 \pm 0.56$ & $12.51 \pm 0.75$ & $12.75 \pm 0.77$ \\
DALI-S        & $3.12 \pm 0.19$ & $\mathbf{2.52 \pm 0.13}$ & $3.05 \pm 0.89$ & $10.68 \pm 0.05$ & $\mathbf{10.45 \pm 1.23}$ & $11.37 \pm 0.33$ \\
LILAC-style   & $2.94 \pm 0.92$ & $3.19 \pm 0.86$ & $\mathbf{2.46 \pm 0.61}$ & $11.14 \pm 0.93$ & $11.54 \pm 0.90$ & $\mathbf{10.22 \pm 0.91}$ \\
cRSSM-D       & $2.84 \pm 0.35$ & $3.05 \pm 0.41$ & $2.55 \pm 0.71$ & $10.48 \pm 0.31$ & $10.68 \pm 1.05$ & $10.38 \pm 0.33$ \\
\bottomrule
\end{tabular}
\end{table}

\textbf{Settings with i.i.d.\ contexts.} Table~\ref{tab:results-iid} confirms
the intended contrast across selected methods. In the hard setting, estimated
headroom remains open and varies for every method, while returns stay within
$2\%$ of the baseline, and the setting falls in the second criterion
case. There, successful control coexists with unresolved predictive
context. Residual context gain remains positive for the inference-based
methods and vanishes only for cRSSM-D, whose linear readout carries the
ground-truth context, as per P4. In the identifiable setting, headroom
collapses under each profiled visitation while the context stays decodable
from the current state and action with $R^2 \geq 0.93$, consistent with P1,
and, with headroom within the permutation null's interval, the setting falls
in the first case. Notably, cRSSM-D leads the returns there while predictive
headroom is negligible. Furthermore, the counterfactual sweep in
Table~\ref{tab:wind-sweep} of Appendix~\ref{app:results} shows that the
trained policy did not move toward states where the context matters less.

\textbf{Settings with non-i.i.d.\ contexts.} As reported in
Table~\ref{tab:results-drift}, the context-blind baseline achieves the lowest
mean startup losses at the low drift rate. Carried state has the largest
transient across rates, separating episode-level latent prediction from simply
continuing the previous episode's state through the reset. However, at the mid
rate, the transition-window agent has the lowest start transient, $0.28 \pm
0.10$ below the baseline, and the lowest start prior loss, $0.32 \pm 0.76$
below (Table~\ref{tab:results-drift-diff} in Appendix~\ref{app:results}). At
the high rate, the chain-latent agent has the lowest start transient and start
prior loss, with differences from the baseline within two standard errors. The
same methods lead in absolute start prior loss at their respective
rates. Together, these profiles are consistent with regime-related predictive
benefits of the conditioning mechanisms under each method's own visitation,
and the drift results fall in the cross-episode criterion case.
Figure~\ref{fig:drift-endpoints} in Appendix~\ref{app:results} depicts the
considered metrics across the methods and drift rates.

\section{Discussion}\label{sec:discussion}
The results support the framework's prediction that the one-step predictive
value of a context mechanism is governed by the information structure of the
setting under the trained agent's visitation in addition to the mechanism's
design. First, the probes tracked exact targets on the calibration setting,
where history recovery reached $0.89$ against the exact $0.99$ (P2), providing
a reference for interpreting their estimates. Next, in the identifiable
i.i.d.\ setting, the headroom collapsed under every profiled visitation while
the context stayed decodable, so return differences there cannot come from
one-step predictive information missing to the transition model and have to
come from states outside the visited support, the finite model, or the
policy's use of the context. In the hard setting, estimated headroom stayed
open across the evaluated methods despite comparable returns, showing that
successful control does not have to resolve predictive context. Finally, the
non-i.i.d.\ results are consistent with regime-related predictive benefits
under each method's own visitation. The contrast between chain-latent
inference and carried state is consistent with the distinction between
predicting episode-level context and preserving recurrent state as per
Equation~\ref{eq:chain-decomposition}.

Our findings carry two implications: i) when Bayes predictive headroom is
zero, any predictive gain under the same visitation reflects a reduction in
finite-model error, and ii) for benchmark construction, the class of a setting
alone does not determine its difficulty. The same task yielded headroom well
above the null in one setting but not in another, which suggests reporting
estimated headroom alongside an MDP's class and its benchmark definition.

\paragraph{Limitations.} In the present analysis all estimated quantities
are properties of a closed loop: each method's visitation shapes the
information available to it, and that information shapes the method in
turn. We deliberately do not compare under a shared reference policy, since no
policy is neutral, random behavior collapses the visited region and competent
behavior encodes the context it has already identified, so any fixed choice
would characterize that policy rather than the learning problem. The policy
dependence is instead the mechanism under study: implicit identification is
produced by the agent's own behavior, and removing it would remove the
phenomenon. It is also the conditioning that matters in practice, since a
deployed mechanism operates under the visitation its training produces, rather
than under a standardized one. The probes also apply to the context-blind
baseline that practitioners build in any case, before any context mechanism is
trained.

The proposed criterion does not cover imagined rollouts, which leave the
visited support. Moreover, the full pipeline depends on the realized context
logged with each episode, while the history gap, the start transient, and the
start prior loss can be estimated without labels. The probe estimates target
quantities at the Bayes level, and probe fitting error remains.

Extending comparison across visitations, through paired counterfactual resets
where simulators permit them, and constructing settings whose context yields
only to deliberately informative actions, are the future directions we
consider most fruitful.

\section{Conclusion}\label{sec:conclusion}
Motivated by real-world control, where the dynamics an agent faces vary and
are rarely observed directly, a growing number of methods target MDPs with
latent dynamics, but nothing tells a practitioner when any of them is needed
for prediction: in practice the choice often follows the environment's class
and the method's design intent. However, we show that both can be unreliable
guides and propose a criterion for which conditioning sets have one-step
predictive information to add. Our analysis establishes, and our experiments
support, that the predictive opportunity available to a context mechanism is
set by where the latent information resides under the trained agent's
visitation, yielding the criterion assessed from the ordinary agent's
trajectories with logged context. Behind it lies a simple observation: an
agent that acts in a world also performs inference about it, and whatever its
behavior identifies along the way is already available to the optimal
predictor. What remains for such additions is the one-step predictive
information their conditioning sets add, which measurement can identify.

\subsection*{AI use statement}
Generative AI tools were used to assist with software implementation,
including environmental benchmarks and analysis scripts, and to edit the
manuscript for language, wording, and phrasing to improve readability. AI was
also used to translate the proposition and its proof into a formal notation
for verification by a theorem prover. The authors reviewed all code, analyses,
and text and take full responsibility for the final content of this work.

\subsection*{Reproducibility statement}
Appendix~\ref{app:controlled_diagnostics} details the MiniGrid and
Point-Mass-Attractor diagnostics, Appendix~\ref{app:experimental_details}
reports the benchmark definitions, evaluation distributions, implementations,
and hyperparameters, while Appendix~\ref{app:calibration} reports the
calibration protocol. The reference
codebase\footnote{\url{https://anonymous.4open.science/r/considering_context/}}
includes all source code, along with configuration scripts and protocols.

\bibliography{main}
\bibliographystyle{iclr2027_conference}

\clearpage
\appendix

\section{Proofs and Technical Remarks}
\label{app:sim-proof}

\begin{proposition*}[Predictive headroom decomposition] Under the conditional-Markov assumption and a common transition distribution,
\begin{equation*}
\mathcal{I}_{\mathrm{emp}} = \mathcal{R}_{\mathrm{M}}-\mathcal{R}_{\mathrm{O}}
= \underbrace{\mathcal{R}_{\mathrm{M}}-
  \mathcal{R}_{\mathrm{H}}}_{\text{headroom closed by history}} +
\underbrace{\mathcal{I}_{\mathrm{hist}}}_{\text{residual Oracle headroom}}.
\end{equation*}
Consequently, $0 \leq \mathcal{I}_{\mathrm{hist}} \leq
\mathcal{I}_{\mathrm{emp}}$.
\end{proposition*}
\begin{proof}
$\mathcal{R}_{\mathrm{M}}=H(s_{t+1} \mid s_t,a_t)$,
$\mathcal{R}_{\mathrm{H}}=H(s_{t+1} \mid \tau_{\leq t},a_t)$, and the
conditional Markov property gives $\mathcal{R}_{\mathrm{O}}=H(s_{t+1} \mid
s_t,a_t,\theta)=H(s_{t+1} \mid \tau_{\leq t},a_t,\theta)$. Subtracting these
risks yields Equation~\ref{eq:headroom-decomposition}.
\end{proof}

\begin{proposition}[Value sensitivity to transition error]
\label{prop:sim}
Let $T_\theta$ and $\widehat T$ be transition kernels over a shared
state--action space, with common bounded reward $r \in [0,R_{\max}]$ and
discount $\gamma \in [0,1)$. For any stationary Markov policy $\pi$ and any
  state $s$,
\begin{equation*}
    \bigl| V^\pi_{T_\theta}(s) - V^\pi_{\widehat T}(s)
    \bigr|\;\le\; \frac{\gamma R_{\max}}{(1 - \gamma)^2}\,\sup_{s,
      a}\sqrt{\tfrac{1}{2}\, D_{\mathrm{KL}}\!\bigl(T_\theta(\cdot \mid s,
      a)\,\|\,\widehat T(\cdot \mid s,a)\bigr)}.
\end{equation*}
\end{proposition}

\begin{proof}
    Write $V_\theta = V^\pi_{T_\theta}$, $\widehat V = V^\pi_{\widehat T}$,
    and $T^\pi(\cdot \mid s) = \sum_a \pi(a \mid s)\, T(\cdot \mid s,a)$ for
    the policy-induced state kernel. Subtracting the two Bellman equations,
    whose shared reward cancels, and inserting $\pm\,\gamma\, T_\theta^\pi
    \widehat V$:
    \begin{equation*}
        V_\theta(s) - \widehat V(s)
        = \gamma\, T_\theta^\pi(\cdot \mid s)
          \bigl[V_\theta - \widehat V\bigr]
        + \gamma\, \bigl[T_\theta^\pi - \widehat T^\pi\bigr]
          (\cdot \mid s)\, \widehat V .
    \end{equation*}
    We bound the two terms. The first is a sup-norm non-expansion:
    \begin{equation*}
        \bigl| T_\theta^\pi(\cdot \mid s)
        \bigl[V_\theta - \widehat V\bigr] \bigr|
        \;\le\; \| V_\theta - \widehat V \|_\infty
        \qquad \text{($T_\theta^\pi(\cdot \mid s)$ is a probability measure).}
    \end{equation*}
    For the second, $T_\theta^\pi - \widehat T^\pi$ has zero total mass, so
    \begin{align*}
        \bigl| [T_\theta^\pi - \widehat T^\pi]\, \widehat V \bigr|
        &= \bigl| [T_\theta^\pi - \widehat T^\pi]\,
          (\widehat V - c) \bigr|
        && \text{(zero mass: subtract any constant $c$)} \\
        &\le \| \widehat V - c \|_\infty \,
          \bigl\| T_\theta^\pi - \widehat T^\pi \bigr\|_1
        && \text{(H\"older's ineq.)} \\
        &\le \frac{R_{\max}}{2(1 - \gamma)} \cdot 2\,
          D_{\mathrm{TV}}(T_\theta^\pi, \widehat T^\pi)
        && \text{($c$ the midrange of
          $\widehat V \in [0, \tfrac{R_{\max}}{1-\gamma}]$)} \\
        &\le \frac{R_{\max}}{1 - \gamma}
          \sup_{s,a} D_{\mathrm{TV}}(T_\theta, \widehat T)
        && \text{(triangle ineq.\ for TV under the shared policy mixture).}
    \end{align*}
    Taking absolute values in the identity, applying the triangle inequality
    with the two bounds, and taking a supremum over $s$:
    \begin{align*}
        \| V_\theta - \widehat V \|_\infty
        &\le \gamma \, \| V_\theta - \widehat V \|_\infty
        + \frac{\gamma R_{\max}}{1 - \gamma}
          \sup_{s,a} D_{\mathrm{TV}}(T_\theta, \widehat T) \\
        &\le \frac{\gamma R_{\max}}{(1 - \gamma)^2}
          \sup_{s,a} D_{\mathrm{TV}}(T_\theta, \widehat T)
        && \text{(rearrange; $\gamma < 1$,
          $\|V_\theta - \widehat V\|_\infty < \infty$)} \\
        &\le \frac{\gamma R_{\max}}{(1 - \gamma)^2}
          \sup_{s,a} \sqrt{\tfrac{1}{2}\,
          D_{\mathrm{KL}}\!\bigl(T_\theta \,\|\, \widehat T\bigr)}
        && \text{(Pinsker's ineq.)}
    \end{align*}
    Since $|V_\theta(s) - \widehat V(s)| \le \| V_\theta - \widehat V
    \|_\infty$ for every $s$, the claim follows.
\end{proof}

\textbf{Jensen-gap identity.} At each fixed $(s,a)$, the integrand of
Equation~\ref{eq:mi-struct} satisfies
\begin{equation}
\mathbb{E}_\theta D_{\mathrm{KL}}(T_\theta\|\bar T_{\mathrm{prior}}) = H(\bar T_{\mathrm{prior}})-\mathbb{E}_\theta H(T_\theta),
\label{eq:jensen}
\end{equation}
with the $(s,a)$ arguments suppressed for clarity. The entropy and
expected-log-loss identities assume finite terms under a common dominating
measure.

\textbf{Remark on recovering the context.} Tasks requiring $\theta$ itself
need the stronger condition $p(\theta\mid\tau_{\leq t})=p(\theta\mid c_t)$ and
trajectory-level information such as $I(\theta;\tau)$. Establishing this
condition from observational trajectories requires assumptions beyond episode
grouping~\citep{hyvarinen2019nonlinear,khemakhem2020variational}.

\section{Controlled Diagnostic Details}
\label{app:controlled_diagnostics}
The controlled diagnostic uses two state spaces to test the same distinction.
MiniGrid makes every $(s,a)$ pair enumerable, so
$\mathcal{I}_{\mathrm{struct}}$ can be evaluated directly and visitation can
be manipulated by setting. Point-Mass-Attractor retains known one-step kernels
in continuous state, where the empirical quantity requires estimating only the
occupancy-induced context posterior. MiniGrid-Resistance supplies a contrast
with low $\mathcal{I}_{\mathrm{struct}}$ within the discrete setting.

\textbf{Shared evaluation protocol.} Context is sampled once per episode and
omitted from the predictive information set, but retained in the logged data
used to estimate conditional information. Because every context-conditional
kernel is known in closed form, no next-state distribution is estimated from
samples: empirical values score each logged transition's exact kernel against
a cross-fitted reference mixture, so every transition contributes and no
count-based filtering occurs. All pooled transitions receive equal weight,
while confidence intervals resample episodes. The standardized
$\mathcal{I}_{\mathrm{struct}}$ uses a fixed $\theta$-independent query
distribution.

\begin{figure}[t]
\centering
\includegraphics[width=\textwidth]{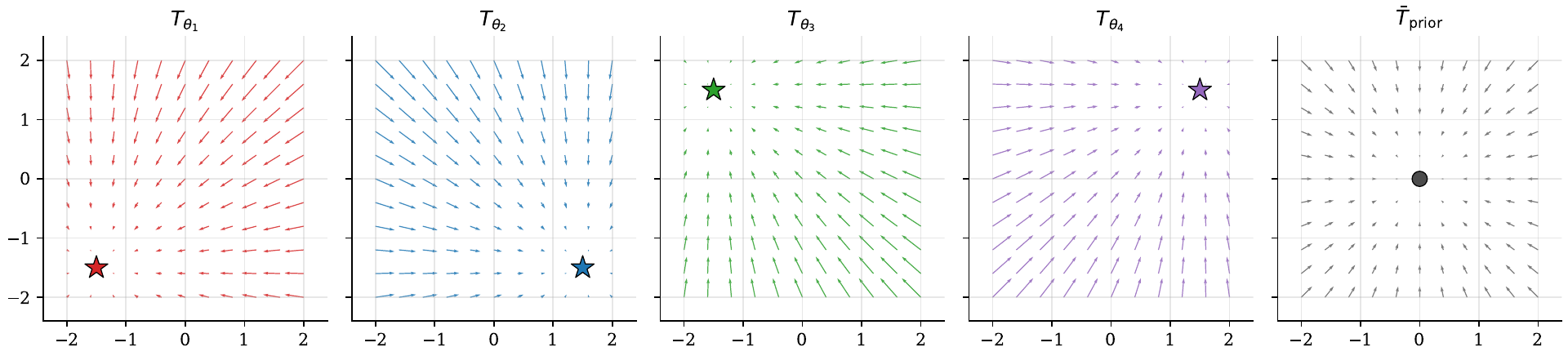}
\caption{Point-Mass-Attractor transition fields. Each episode uses one hidden
  attractor $\theta_i$, inducing a distinct field $T_{\theta_i}$; prior
  averaging produces a center-seeking field $\bar T_{\mathrm{prior}}$ that
  governs no episode.}
\label{fig:attractor_force_field}
\end{figure}

\begin{figure}[t]
\centering
\includegraphics[width=0.8\textwidth]{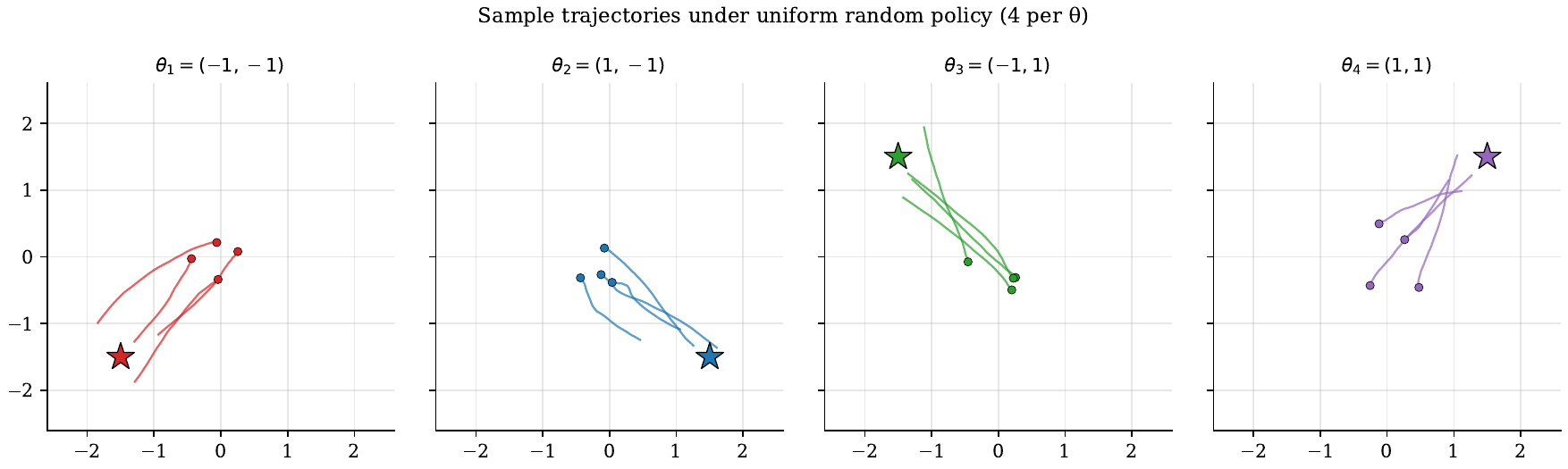}
\caption{Point-Mass-Attractor trajectories under uniform random actions. The
  hidden attractor is fixed within an episode, causing visited states to become
  informative about context.}
\label{fig:attractor-env}
\end{figure}

\textbf{Discrete state: MiniGrid.} \label{app:minigrid_diagnostic} Both
MiniGrid families use the fully observed Markov state $(x,y,d)$ on the
\texttt{Empty-8x8} layout and leave the goal and reward unchanged. In
MiniGrid-Wind, the episode context is $\mathbf v=(v_x,v_y)$ with independent
coordinates sampled from $\mathcal{U}[-v_{\max},v_{\max}]$. After the native
action, each axis receives drift $\operatorname{sign}(v_j)$ with probability
$|v_j|$ and zero drift otherwise. Invalid displacement into a wall is
blocked. MiniGrid-Resistance samples
$\epsilon\sim\mathcal{U}[0,\epsilon_{\max}]$; an action takes effect with
probability $1-\epsilon$ and otherwise leaves the state unchanged. Zero
perturbation recovers the base environment in both families.

For the standardized reference, $\mu$ is uniform over interior position,
direction, and the seven MiniGrid actions. The matched-query reference instead
weights those same local calculations by $(s,a)$ frequencies in the collected
trajectories. Figure~\ref{fig:diagnostic-sweeps} reports the full $v_{\max}$
and $\epsilon_{\max}$ sweeps.

Each setting uses $5{,}000$ uniform-random episodes of at most $256$ steps.
The $\mathcal{I}_{\mathrm{struct}}$ calculation enumerates $\approx 1{,}000$
interior $(s,a)$ queries and integrates the continuous context prior by
quadrature. The empirical estimator is a known-kernel occupancy plug-in: each
logged transition contributes the exact Kullback--Leibler divergence between
its context-conditional kernel and a kernel mixture over the transitions
observed at the same $(s,a)$ cell. The reported value cross-fits the mixture
leave-one-episode-out with one pseudo-count of uniform smoothing on the cell's
outcome support, which makes it a conservative upper estimate of the true
headroom in expectation; the within-sample mixture value is always within 2\%
of it. Percentile CIs come from 200 episode-level bootstrap resamples with
full re-estimation of all mixtures. The sweeps use
$v_{\max}\in\{0,0.1,0.3,0.5,0.7,1.0\}$ and
$\epsilon_{\max}\in\{0,0.1,0.3,0.5,0.7,0.9\}$.

\textbf{Continuous state: Point-Mass-Attractor.}
\label{app:attractor_diagnostic}
Point-Mass-Attractor is a unit-mass two-dimensional system integrated at
$\Delta t=0.05$. The hidden attractor is selected uniformly from the four
corners $(\pm1.5,\pm1.5)$ and held fixed for the episode. At each step, the
agent applies $\mathbf a_t\in[-1,1]^2$ and the attractor contributes force
$-k(\mathbf p_t-\theta)$. Independent Euler--Maruyama increments add variance
$\sigma^2\Delta t$ to each velocity coordinate and $\sigma^2\Delta t^3$ to
each position coordinate. We use $k=0.3$ and $\sigma=0.1$ for the
representative setting.

Episodes begin with $\mathbf p_0\sim\mathcal{U}[-0.5,0.5]^2$ and $\mathbf
v_0=\mathbf{0}$, and run for 50 steps
(Figure~\ref{fig:attractor-env}). Positions are confined to $[-2.5,2.5]^2$ by
elastic reflection with velocity damping 0.5. Away from the boundary, the
conditional kernel is diagonal Gaussian, with four means separated by a
query-independent context displacement. Thus, the uniform and matched-query
values of $\mathcal{I}_{\mathrm{struct}}$ in the interior coincide
exactly. Per coordinate, the context sign defines a binary-input Gaussian
channel, so $\mathcal{I}_{\mathrm{struct}}$ can be computed by one-dimensional
Gauss--Hermite quadrature. A standard bound shows that boundary reflection has
a negligible effect at the occupancies used here.

Each $k$ setting uses $5{,}000$ episodes, with
$k\in\{0.05,0.1,0.2,0.3,0.5\}$. The empirical estimate scores every logged
transition's exact context-conditional Gaussian kernel, including the exact
single-reflection boundary densities, against the four-kernel mixture weighted
by a cross-fitted $k$-NN context posterior (episode-level folds; the number of
neighbors is selected by cross-validated log loss). For each candidate
neighbor count the fitted mixture is a valid state-measurable predictor, so
the estimate is biased upward in expectation up to a small model-selection
effect. Percentile CIs come from $1{,}000$ episode-level bootstrap resamples
of per-episode contributions. As an internal certificate, transitions
immediately after a zero-velocity reset in the intervention datasets below
have an exactly known uniform context posterior; the fitted estimator
reproduces the exact stratum contribution within about 1.5\% (for example,
0.4845 versus 0.4834 at reset interval two).

\textbf{Locating the visitation effect.} Two checks vary the information
carried by the current state. First, state projection removes coordinates from
the conditioning variable while retaining the full next state. MiniGrid uses
$(x,y,d)$, direction alone, or no state; Point-Mass-Attractor uses $(\mathbf
p,\mathbf v)$, position, a $5\times5$ position bin, or no state. Coarsening
therefore changes the estimand. The MiniGrid projections are evaluated at a
declared discretization of ten bins per context coordinate, with both the
projected conditional and marginal formed as occupancy-weighted exact-kernel
mixtures, cross-fitted leave-one-episode-out. The Point-Mass-Attractor
projections use a cross-fitted classifier contrast: with four known context
values, the estimand equals the difference between the context entropy given
the projected state and given the projected state, action, and next state,
each estimated by a cross-fitted $k$-NN class posterior (for the no-state
projection the first term is the exact design prior $\ln 4$, making the
estimate conservative). Every projection continues to condition on action.

Second, resets intervene on visitation while preserving the episode context.
MiniGrid resets position and direction before the current state is recorded,
at intervals of ten, five, or two steps. Point-Mass-Attractor resets position
at intervals of 25, ten, five, or two steps. Its control either resets
velocity to zero or preserves it across the position reset. All subsequent
transitions are recorded from the post-reset state.

\begin{figure}[t]
\centering
\includegraphics[width=\linewidth]{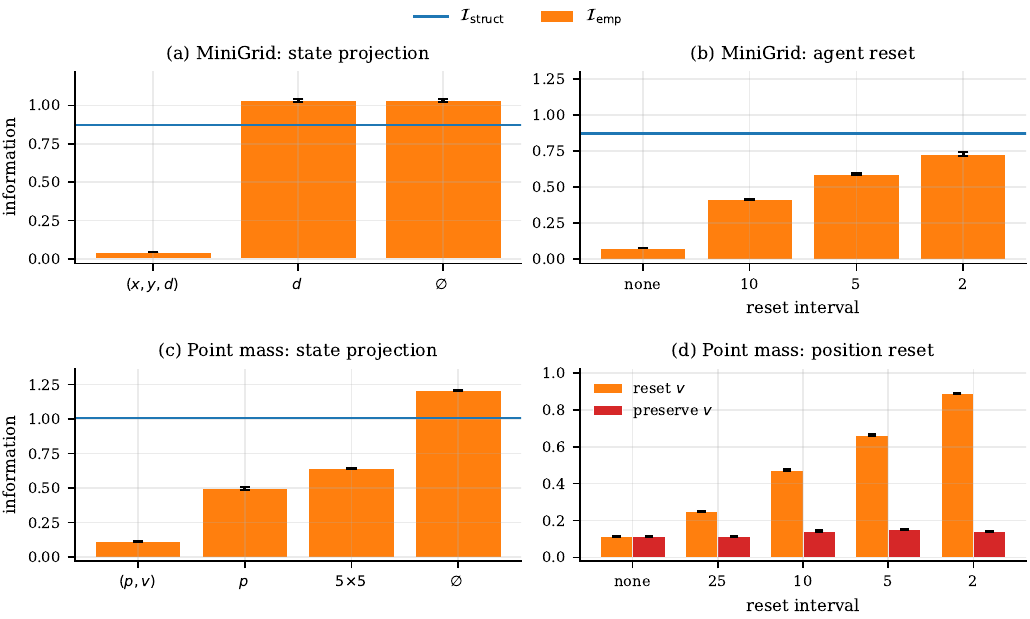}
\caption{Parallel checks of state-mediated identification. State projection
  removes current-state information in MiniGrid (a) and Point-Mass-Attractor
  (c). State resets weaken context-dependent visitation in MiniGrid (b). In
  Point-Mass-Attractor (d), headroom rises when velocity is reset and remains
  near baseline when velocity is preserved. Error bars denote $95\%$ bootstrap
  CIs. Where shown, blue lines are full-state structural references, not
  bounds after the information set or query distribution changes.}
\label{fig:diagnostic-interventions}
\end{figure}

The projection results, depicted in
Figure~\ref{fig:diagnostic-interventions}, locate context information already
present in the observed state. Removing position from MiniGrid raises estimated
$\mathcal{I}_{\mathrm{emp}}$ by about 2{,}300\%. In Point-Mass-Attractor,
progressively removing velocity and position raises it by about 980\%. The
resets partially restore headroom without imposing the standardized query
distribution. Relative to baseline, estimated MiniGrid headroom increases by
430\%, 658\%, and 833\% as the reset interval decreases. In
Point-Mass-Attractor, resetting both position and velocity raises headroom by
as much as 698\%, increasing monotonically as the reset interval shrinks,
whereas preserving velocity keeps it within 35\% of baseline.

\section{Experimental Details}
\label{app:experimental_details}

\textbf{Benchmark environments.} The i.i.d.\ settings use CARL DMC Walker
Walk~\citep{benjamins2022contextualize}. The identifiable setting samples a
latent horizontal wind uniformly per episode, with nonzero fluid density so the
wind exerts force on the body. The hard setting samples a latent
actuator-strength scale uniformly around nominal, with gravity fixed at nominal.
An $84$-context benchmark panel is held out for the final evaluation. The drift setting follows the registered RWRL
protocol~\citep{dulacarnold2021challenges}: one Walker parameter performs a
random walk across episodes at three step sizes on frozen realizations shared
by every method, with the step sizes chosen so that consecutive-episode
autocorrelation stays above $0.5$ at the highest rate.

\textbf{Agents and conditioning interfaces.} All agents use the common
DreamerV3 hyperparameters~\citep{hafner2025mastering} in
Table~\ref{tab:common-hyperparameters}. Every method shares this backbone and differs
only in its context configuration, which selects one information path while
leaving the training stack unchanged: the context-blind baseline receives no
context, DALI-S infers one from a recent window, cRSSM-D receives the true
context through deep conditioning, the carried-state variant preserves the
recurrent state across episode boundaries, and the chain-latent agent
maintains a cross-episode latent inspired by
LILAC~\citep{pmlr-v139-xie21c}. The true-context method receives the context
as a separate normalized observation with explicit bounds, and the ordinary
observation encoder always excludes it. DreamerV3 already carries a recurrent
within-episode state, so its comparison with DALI isolates the value of a
dedicated recent-history pathway while holding that recurrent backbone fixed.

DALI-S~\citep{roder2025dynamics} uses a causal $64$-step state--action encoder
with one head and width $256$, an $8$-dimensional context, and a two-layer
width-$128$ forward model, with unit loss scale and context learning rate
$10^{-4}$. The inferred context is detached before it enters the world model,
its heads, the actor, and the critic, and the cross-modal loss is omitted.
cRSSM-D~\citep{prasanna2024dreaming} receives true normalized context through
the transition and prediction/control heads and adds no learned context
encoder. The chain-latent baseline uses an $8$-dimensional context, a
two-layer width-$128$ episode encoder, and a width-$64$ GRU over the preceding
$128$ episode latents, trained at $10^{-4}$ with chain and consistency loss
scales $1$ and $0.1$. All comparators are implemented within the same
DreamerV3 codebase using the conditioning interfaces described above.

\textbf{Training and evaluation.} Return entries are IQM over paired
evaluation contexts after within-seed aggregation, with CARL return on the
original $0$ to $1000$ scale. Confidence intervals use $10{,}000$ paired,
seed-level bootstrap samples. Contexts are not treated as independent
units. On the drift setting the primary endpoint is the start transient over
the first $25$ steps, computed post hoc from final checkpoints.
Drift entries are means of seed means with sample standard deviations across
seed means. Return and per-seed collapse counts are secondary.

Runs use one $24$ GB NVIDIA L4.
A $10^6$-step run takes about $9$ hours, with $6$--$7$ GB peak allocated
memory. Checkpoint evaluation freezes the model and policy and performs no
learning updates, while only explicitly carried recurrent state persists
across episode resets.

\textbf{Probe implementation.} Probe architectures and training settings are
reported beside the agent settings in
Table~\ref{tab:common-hyperparameters}. Scoring is cross-fitted at episode
level, so every transition is scored by a probe that never trained on its
episode. The permutation null re-pairs episodes with contexts drawn from other
episodes, and the decodability readout uses the feedforward architecture with
a squared-error objective to predict the context from the current state and
action; its cross-fitted $R^2$ is reported. Learned-state inputs are extracted
before the target transition is observed. The shared Gaussian likelihood
defines the predictive risk used in continuous benchmark
settings. Episode-start estimates restrict every quantity to the beginning of
each episode. During training, a frozen panel of open-loop trajectories, one
per evaluation context with autocorrelated exploration actions of length $64$,
tracks prior prediction loss at every evaluation point.

\begin{table}[t]
\centering
\caption{Agent and probe hyperparameters.}
\label{tab:common-hyperparameters}
\begin{minipage}[t]{0.49\linewidth}
\centering
\scriptsize
\begin{tabular}{@{}ll@{}}
\toprule
Agent setting & Value \\
\midrule
Backbone & DreamerV3, block-GRU \\
Parameters (Walker baseline) & $9.7$M ($4.3$M transition) \\
Deterministic state & $2048$ \\
Stochastic state & $32\times16$ categorical \\
RSSM hidden width & $256$ \\
Encoder / decoder & $3$ layers, width $256$ \\
Actor / critic & $3$ layers, width $256$ \\
Activation / normalization & SiLU / layer norm \\
Batch / sequence & $16 / 64$ \\
Replay capacity / ratio & $5\times10^5 / 512$ \\
Optimizer / learning rate & LaProp / $4\times10^{-5}$ \\
Betas / warmup & $(0.9,0.999) / 1000$ \\
Weight decay & none \\
Gradient clipping & adaptive, $0.3$ \\
KL free threshold & $1.0$ \\
Imagination horizon & $15$ \\
Discount / $\lambda$ & $1-1/333$ / $0.95$ \\
Actor entropy scale & $3\times10^{-4}$ \\
Action repeat & $1$ \\
\bottomrule
\end{tabular}
\end{minipage}\hfill
\begin{minipage}[t]{0.48\linewidth}
\centering
\scriptsize
\begin{tabular}{@{}ll@{}}
\toprule
Probe setting & Value \\
\midrule
Feedforward architecture & $2$-layer MLP, width $256$ \\
Recurrent architecture & $1$-layer GRU, width $128$ \\
Activation & SiLU (feedforward) \\
Predictive output & diagonal Gaussian \\
Optimizer / learning rate & AdamW / $10^{-3}$ \\
Training updates & $4{,}000$ \\
Batch & $1{,}024$ transitions / $32$ episodes \\
Feedforward weight decay & $\{10^{-4},3\times10^{-2},1\}$ \\
Recurrent weight decay & $\{10^{-4},3\times10^{-2}\}$ \\
Cross-fitting & $5$ episode-level folds \\
History span & $16$ steps / full episode \\
Episode-start window & first $25$ steps \\
Selection / calibration & validation split \\
\bottomrule
\end{tabular}
\end{minipage}
\end{table}

\section{Analytical Calibration Details}
\label{app:calibration}

The considered benchmark settings do not offer exact information targets,
while the diagnostic environments have known kernels but no analytic probe
targets, since the proposed probes have closed-form equivalents only under
linear-Gaussian dynamics. To this end, we calibrate the measurements using the
simplest such system.

\textbf{Environment and analytical quantities.} Figure~\ref{fig:calibration-environment}
shows the calibration task: a two-dimensional stabilization point mass whose
state $(x,v,d)$ is fully
observed and whose hidden regime $z\in\{-1,+1\}$ follows a slow-switching
Markov chain across episodes with stay probability $0.9$. The disturbance
follows $d_{t+1}=\rho R(z\omega)d_t+\sigma\sqrt{1-\rho^2}\epsilon_t$, with
rotation $R$, $\rho=0.9$, $\omega=\pi/4$, $\sigma=0.3$, and independent
$\epsilon_t\sim\mathcal{N}(0,I_2)$. Velocity and position update as
$v_{t+1}=0.9v_t+0.5a_t+d_{t+1}$ and $x_{t+1}=x_t+v_{t+1}$; the reward is
$-(\lVert x\rVert^2+0.01\lVert a\rVert^2)$. At reset, $x_0=v_0=0$ and
$d_0\sim\mathcal{N}(0,\sigma^2I_2)$ independently of the regime. The first
regime is drawn uniformly and episodes last $200$ steps. Conditional
disturbance densities and disturbance-history posteriors are available
analytically: their averaged scores are evaluated by numerical integration or
sample averages.

\begin{figure}[t]
\centering
\includegraphics[width=\textwidth]{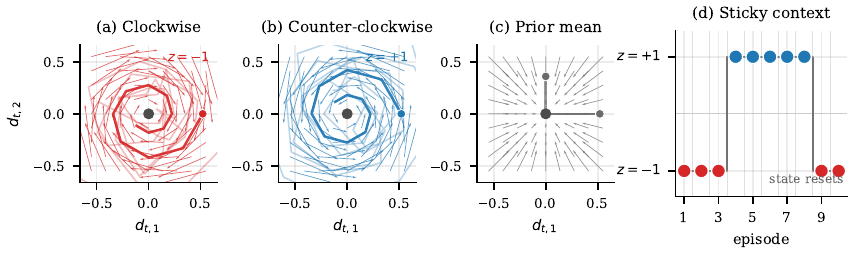}
\caption{Analytical calibration environment. The hidden regime determines the
  rotation of the autoregressive disturbance. The physical state resets
  between episodes while the regime follows a Markov chain. The prior mixture
  averages the two conditional disturbance fields but is not a realized
  regime.}
\label{fig:calibration-environment}
\end{figure}

\textbf{Evaluation protocol.} We evaluate DreamerV3, DALI-S, carried state,
the LILAC-inspired chain-latent agent, and cRSSM-D using the common agent and
probe settings in Table~\ref{tab:common-hyperparameters}. Each method uses
seeds $0$--$2$, trains for $10^5$ environment steps, and is evaluated on $10$
episodes. The primary setting uses the sticky context chain. An
independent-draw control sets the stay probability to $0.5$, and a
short-episode variant truncates episodes at $12$ steps to leave a nonzero
residual.

\textbf{Score-based targets and information ceilings.} The probe targets are
ideal \emph{diagonal-Gaussian score gaps}, rather than mutual
information. Given $(S_t,A_t)$, the six observed next-state coordinates share
one two-dimensional innovation $d_{t+1}$ across the position, velocity, and
disturbance blocks, so the conditional covariance has rank two and no
six-dimensional density exists; the coordinatewise diagonal-Gaussian score is
the well-defined reference, and the three identical blocks contribute the
factor of three in its gap. With uniform regime weights the population gap is
\[
G_0=\tfrac{3}{2}\,\mathbb{E}\sum_{j=1}^{2}
\log\!\Big[1+\tfrac{(\mu_{+}-\mu_{-})_j^2}{4\tau^2}\Big]=2.518\ 
\quad (\tau^2=\sigma^2(1-\rho^2)),
\]
evaluated by quadrature, whereas the sampled evaluation in
Table~\ref{tab:calibration} gives $2.51\pm0.05$. Because
$w(1-w)\le\tfrac14$ for any regime probability $w$, $G_0$ upper-bounds the
state-conditional score gap. The corresponding \emph{information} quantities
are much smaller: the binary regime carries at most $\ln 2\approx0.693$ nats,
and the stationary disturbance-mixture information is
$I(Z;D_{t+1}\mid D_t)=0.507$, which upper-bounds the memoryless context
information $I(Z;S_{t+1}\mid S_t,A_t)$ here because $(S_t,A_t)$ determines
$D_t$ and $D_{t+1}$ is independent of $(S_t,A_t)$ given $(Z,D_t)$.

\textbf{Calibration results.} The estimated score gap in
Table~\ref{tab:calibration} tracks the exact target at a stable ratio in the
primary and independent-draw settings, which characterizes its attenuation on
the calibration environment. The unadjusted reference upper-bounds the population context gain
for moment-matched diagonal-Gaussian predictors conditioned on state and
action, by the argument above. The shuffled null stays near zero, while
history recovery reaches $0.89 \pm 0.03$ of the $0.99$ reference computed from
the exact disturbance-history filter, giving a recovered fraction for the
nested disturbance-only information sets rather than a certified bound on the
full on-policy history optimum. The estimated sufficiency gap is near zero for
the observed-context agent, concentrates at episode start for every inference
method, and is largest for carried state, whose recurrent state continues the
previous episode through the reset. The chain classes have no advantage at
episode start, consistent with the small cross-episode increment in the closed
form.

The short-episode variant yields an exact per-transition residual of
$0.24$; since the scoring window covers all twelve steps of this variant,
start-window and full-episode averages coincide there, and the probes
reproduce the residual as a sufficiency gap that persists for the
context-blind model and nearly vanishes for the observed-context one. Under
the independent-draw control the exact cross-episode increment is zero and the
chain classes remain indistinguishable from within-episode ones, consistent
with P3. The characterized ratio reproduces at $0.51 \pm 0.04$. 

These
controls support using the probes as comparative diagnostics at the operating
episode length; the attenuation ratio is characterized on this environment
rather than assumed to transfer elsewhere.

\begin{table}[t]
\centering
\scriptsize
\caption{Probe calibration on the analytical point-mass environment. The
  context score gap is a diagonal-Gaussian log-score difference: the exact
  column is a sampled evaluation of the $2.518$-nat population gap, and the
  estimate column reports the estimate-to-target ratio. History recovery
  reports a recovered fraction and the shuffled null a probe-risk difference.
  Entries are IQM $\pm$ the larger 95\% bootstrap CI half-width.}
\label{tab:calibration}
\begin{minipage}[t]{0.30\linewidth}
\centering
\begin{tabular}{@{}lcc@{}}
\multicolumn{3}{@{}l}{\emph{(a) Probe calibration}} \\
\toprule
Quantity & Exact & Est. \\
\midrule
Context score gap & $2.51 \pm 0.05$ & $0.48 \pm 0.03$ \\
History recovery & $0.99$ & $0.89 \pm 0.03$ \\
Shuffled null & $0$ & $-0.04 \pm 0.02$ \\
\bottomrule
\end{tabular}
\end{minipage}
\end{table}

\section{Additional Results}
\label{app:results}

\textbf{Counterfactual wind sweep.} For the identifiable setting we
reconstruct the simulator state from the $24$-dimensional observation at each
recorded transition, set the wind to each value of the grid $\{-4,-2,0,2,4\}$,
apply one control step with the recorded action, and measure the spread of the
resulting next states across the grid, reported in Table~\ref{tab:wind-sweep}
as the median and the 90th percentile over transitions, with medians over the
first $25$ steps and the later steps of each episode separately. The replay
error between the reconstructed and the recorded next state at the recorded
wind has median $0$ and maximum $0.107$. The trained-policy row uses
$45{,}000$ transitions from one retrained seed of the context-blind baseline
at $5\times10^5$ steps, and the early-training row uses $32{,}000$ transitions
from the first $32$ training episodes. Across the grid, the next state
following each of the trained policy's visited states varies by a median
spread of $0.068$, about $2.5\%$ of the typical one-step state change, and the
early-training episodes give $0.075$, about $10\%$ above. Hence, the policy
did not move toward states where the wind matters less.

\begin{table}[!htb]
\centering
\caption{Counterfactual sweep on the identifiable setting: spread of the next
  state across the wind grid at the visited states.}
\label{tab:wind-sweep}
\scriptsize
\begin{tabular}{@{}lrrrrr@{}}
\toprule
Visitation & Transitions & Median spread & 90th percentile & Steps $<25$ (median) & Steps $\geq 25$ (median) \\
\midrule
trained policy & $45{,}000$ & $0.068$ & $0.149$ & $0.067$ & $0.068$ \\
early training (first 32 episodes) & $32{,}000$ & $0.075$ & $0.128$ & $0.076$ & $0.075$ \\
\bottomrule
\end{tabular}
\end{table}

\begin{table}[!htb]
\centering
\scriptsize
\caption{Differences from DreamerV3 in start transient and start prior loss
  relative to baseline, with two-sample standard errors.}
\label{tab:results-drift-diff}
\begin{tabular}{@{}lcccccc@{}}
\toprule
& \multicolumn{3}{c}{$\Delta$ Start transient}
& \multicolumn{3}{c}{$\Delta$ Start prior loss} \\
\cmidrule(lr){2-4}\cmidrule(l){5-7}
Method
& \rotatebox{45}{Low} & \rotatebox{45}{Mid} & \rotatebox{45}{High}
& \rotatebox{45}{Low} & \rotatebox{45}{Mid} & \rotatebox{45}{High} \\
\midrule
Carried state & $+1.41 \pm 0.26$ & $+1.12 \pm 0.32$ & $+1.15 \pm 0.45$ & $+3.14 \pm 0.45$ & $+1.74 \pm 0.43$ & $+1.51 \pm 0.50$ \\
DALI-S        & $+0.47 \pm 0.23$ & $-0.28 \pm 0.10$ & $+0.26 \pm 0.64$ & $+0.53 \pm 0.37$ & $-0.32 \pm 0.76$ & $+0.13 \pm 0.42$ \\
LILAC-style   & $+0.30 \pm 0.46$ & $+0.40 \pm 0.39$ & $-0.33 \pm 0.47$ & $+0.98 \pm 0.55$ & $+0.77 \pm 0.48$ & $-1.02 \pm 0.55$ \\
cRSSM-D       & $+0.20 \pm 0.26$ & $+0.25 \pm 0.20$ & $-0.23 \pm 0.50$ & $+0.33 \pm 0.39$ & $-0.09 \pm 0.54$ & $-0.86 \pm 0.40$ \\
\bottomrule
\end{tabular}
\end{table}

\begin{figure}[!htb]
\centering
\includegraphics[width=\linewidth]{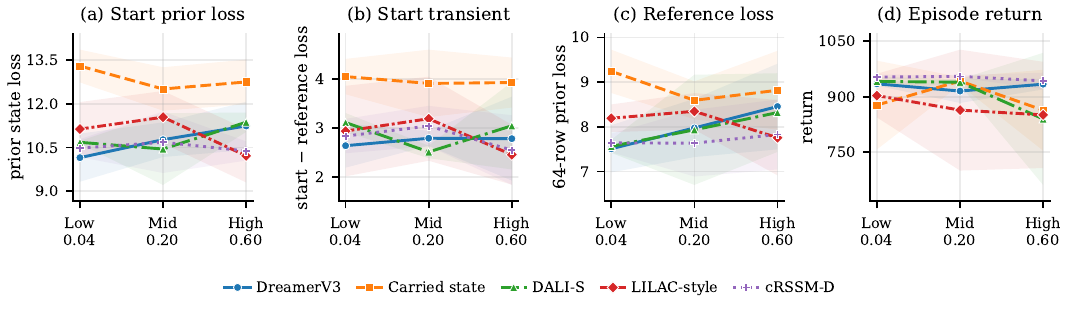}
\caption{Drift metrics at the three rates on frozen realizations shared across
  methods. Lower is better in (a)--(c), higher in (d); lines are seed means
  and bands show sample stdev. The lowest mean startup losses shift from the
  context-blind baseline to the transition-window and chain-latent agents as
  drift increases.}
\label{fig:drift-endpoints}
\end{figure}

\end{document}